\documentclass[preprint,12pt]{elsarticle}

\usepackage{amssymb}
\usepackage{amsmath,amssymb,amsfonts}
\usepackage{algorithm}
\usepackage{algpseudocode}
\usepackage{graphicx}
\usepackage{textcomp}
\usepackage{amssymb}
\usepackage{subfigure}
\usepackage{graphicx}
\usepackage{url}
\usepackage{amsmath}
\usepackage{setspace}
\usepackage{multirow}
\usepackage{booktabs}
\newtheorem{theorem}{Theorem}

\newtheorem{proof}{Proof}[section]
\begin{document}

\begin{frontmatter}



\title{Efficient Multi-robot Exploration via Multi-head Attention-based Cooperation Strategy}

\author[1]{Shuqi Liu}
\author[2]{Zhaoxia Wu}

\address[1]{National Key Laboratory of Big Data Management and Analysis, College of Computer and Science, Northeastern University, Shenyang 110004, China; (e-mail: {1701742@stu.neu.edu.cn})}

\address[2]{School of Control Engineering, Northeastern University at Qinhuangdao, Qinhuangdao 066004, China; (e-mail: {1000429@neuq.edu.cn}) }

\begin{abstract}
The goal of coordinated multi-robot exploration tasks is to employ a team of autonomous robots to explore an unknown environment as quickly as possible. Compared with ``human-designed'' methods, which began with heuristic and rule-based approaches, ``learning-based'' methods enable individual robots  to learn sophisticated and hard-to-design cooperation strategies through deep reinforcement learning technologies. However, in decentralized multi-robot exploration tasks, “learning-based” algorithms are still far from being universally applicable  to the continuous space due to the difficulties associated with area calculation and reward function designing; moreover, existing ``learning-based'' methods encounter problems when attempting to balance the ``historical trajectory'' issue and target area conflict problem. Furthermore, the scalability of these methods to a large number of agents is poor because of the exponential explosion problem of state space. Accordingly, this paper proposes a novel approach – Multi-head Attention-based Multi-robot Exploration in Continuous Space (MAMECS) – aimed at reducing the state space and automatically ``learning'' the cooperation strategies required for decentralized multi-robot exploration tasks in continuous space. Computational geometry knowledge is applied to describe the environment in continuous space and to design an improved reward function to ensure a superior exploration rate. Moreover, the multi-head attention mechanism employed helps to solve the ``historical trajectory'' issue in the decentralized multi-robot exploration task, as well as to reduce the quadratic increase of action space.

\end{abstract}

\begin{keyword}
Multi-robot exploration \sep deep reinforcement learning \sep multi-head attention mechanism \sep continuous space


\end{keyword}

\end{frontmatter}


\section{Introduction}
\label{sec:introduction}

The problems associated with exploring an unknown environment using a team of robots are among the fundamental problems in mobile robotics. These problems arise in a wide range of applications, including disaster rescue, planetary exploration, reconnaissance and surveillance \cite{sheng2006distributed, burgard2005coordinated}. The key question during exploration is that of how to figure out each agent's next move so that the overall mission time is minimized and the exploration rate is maximized. We here focus on a sub-problem of robotic exploration, namely that of decentralized multi-robot exploration tasks, where robots make their own decisions without a centralized controller. In order to devise cooperation strategies, robots need to broadcast their local observations and historical trajectories by communicating with each other, which allows more information about the environment to be acquired. 

To this end, several multi-agent exploration approaches have been developed. The original approaches, such as the frontier-based approach \cite{yamauchi1998frontier,faigl2015benchmarking,sharma2016frontier}  and the cost-utility approach  \cite{burgard2000collaborative, colares2016next}, were designed by experts based on cooperation strategies including explicit communication and action rules. However, many real-world applications have proven too complex to be dealt with efficiently by human-designed strategies. Moreover, these approaches also find it difficult to cope with the ``historical trajectory” issue. Most ``pre-designed” methods assume that the current robot only communicates with nearby robots; however, more distant robots that have explored the surrounding areas also need to be involved in the communications network  to avoid repeated exploration. 

Recent work in this area has attempted to combine the strengths of deep learning techniques with the control policies for robotics applications \cite{pinto2016supersizing,chen2014door,shvets2018automatic}. In particular, deep reinforcement learning (DRL) methods allow multiple agents to autonomously learn the required cooperation strategy \cite{kretzschmar2016socially,gu2017deep,kahn2018self}. Therefore, these learning-based approaches can resolve the difficulties associated with developing precise and complicated control strategies for each move, and thus achieve more flexible and effective performance in complex scenarios.

Despite this progress, however, algorithms for multi-agent exploration are still far from being universal (to the continuous space) and scalable (to a larger number of agents). Various previous works \cite{yamauchi1998frontier,carrillo2015autonomous,mox2018information} have modeled the exploration environment as a discrete space in which agents' actions are restricted to their surrounding grids. When extending the task into continuous space, however, it is hard to design an accurate reward function based on the historical trajectories, meaning that some areas may be ignored or repeatedly explored. In addition, the maximum number of agents is limited in previous works of this kind, as the action space increases exponentially with the number of agents. Although recent single-head attention-based methods have shown great potential in multi-agent cooperation tasks by focusing only on the relevant agent, they are still simple and limited compared with multi-head attention mechanisms, as each attention head used in the multi-head methods can focus on a different weighted mixture of agents (e.g. locations, historical trajectories, etc.).

Accordingly, our proposed approach, Multi-head Attention-based Multi-robot Exploration in Continuous Space (MAMECS), extends these prior works in several directions. We model the environment as a continuous space in which agents can move to an arbitrary point at every step. Computational geometry knowledge is applied to describe the environment and design an improved reward function. Inspired by team performance in real-world applications, each team member tends to focus only on the teammates that exist in a cooperative or competitive relationship with itself; we thus learn the multi-agent cooperation strategy through a multi-head attention-based critic. Therefore, each agent is aware of which other agents it should be paying attention to rather than simply considering all agents at every time step. Moreover, the quadratic increase in the action space is sharply reduced due to the selected attention mechanism, meaning that the number of agents involved can be increased.

We have validated our approach MAMECS on the typical multi-robot exploration task. Extensive experiments have shown that MAMECS can perfectly fit the continuous space, effectively extend the total number of agents and improve exploration performance compared with previous works. The rest of this paper is organized as follows. In section \ref{sec:related_work}, we discuss related work, followed by a detailed description of our approach in section \ref{sec:our_approach}. We report experimental studies in section \ref{sec:experiments} and conclusion in section \ref{sec:conclusion}.
%%%%%%%%%%%%%%%%%%%%%%%%%%%%%%%%%%%%%%%%%%





\section{Related Work}
\label{sec:related_work}

Our approach MAMECS aims to solve the multi-robot exploration task through multi-head attention-based reinforcement learning in continuous space. Therefore, we mainly focus on two research fields: traditional ``human-designed'' methods for multi-robot exploration and the ``learning-based'' multi-robot exploration approaches. 

\subsection{``Human-designed'' Methods for Multi-robot Exploration}

Multi-robot exploration is a fundamental robotic problem, which employs a team of autonomous robots to explore an unknown environment with obstacles. Most early works started with heuristic and rule-based approaches. Yamauchi's Frontier Based Exploration Using Mobile Robots \cite{yamauchi1998frontier} is a foundational paper used by many successful approaches. In this approach, each robot makes the assignment that maximizes the joint utility to the frontiers and navigate to the nearest unvisited frontier. 

Market-based approaches \cite{zlot2002multi} employs the concept of frontier cells and utility in a market environment to produce complex coordinated strategy in multi-robot exploration. Spanning tree coverage approaches \cite{gabriely2001spanning} adapted the single robot complete coverage algorithm to multi-robot scenario. Each robot is assigned a part of the constructed spanning tree and covers the section in a counterclockwise fashion. Recent approaches \cite{andre2016collaboration, corah2017efficient, corah2019communication} focuses more on the mutual information for ranging sensors, and they attempt to maximize mutual information directly. The above methods are all based on precisely designed rules, and they should take all the details and situations into account. Therefore, these ``pre-designed'' cooperation methods will  perform poorly especially for partial observation task. It is extremely hard for human to design effective strategies only based on the local view of the whole environment.

\subsection{Learning-based Methods in Multi-robot Exploration}

Deep Reinforcement Learning has been proved to be effective for enabling sophisticated and hard-to-design behaviors of robot individuals \cite{kretzschmar2016socially,kahn2018self}. For the multi-robot exploration task, \cite{geng2018learning} proposes a learning-based method to enable the robots to actively learn the cooperation strategies as well as the action policies. Their method is robust enough to handle complex and dynamic environments and beats the performance of several ``human-designed'' methods. The communication model used in \cite{geng2018learning} is CommNet \cite{sukhbaatar2016learning}, which simply averages the communication message to realize coordination. \cite{geng2019learning, liu2019learning} improves the communication process by introducing the attention mechanism, which can precisely calculate whether the communication is necessary for each pair of agents in the exploration scenario. The attention mechanism enables the agents to communicate only with the necessary partners and further improves the cooperation performance. However, the above methods simply model the environment by occupancy grids, which is discrete and easy to represent the information such as historical trajectories. 

In this paper, we focus on the multi-robot exploration task in continuous space, which is extremely difficult due to the reason of reward function designing. Besides, we exploit the multi-head attention mechanism and each head can focus on a different weighted mixture of agents (i.e., the locations, the historical trajectories). Furthermore, our method is more flexible than the existing learning-based methods, which can further increase the number of agents in the limited action space and is closer to reality.

\section{Our Approach}
\label{sec:our_approach}

The exploration rate and scalability of existing multi-robot exploration methods are hard to satisfy the requirement of realistic applications. Therefore, it is meaningful to improve the exploration rate and increase the number of involving agents. In this section, we introduce our MAMECS method from the following aspects: the basic framework, learning the shared attentive critic, continuous environment modeling, design of entropy-oriented reward function and the exploration rate-based training approach.

\subsection{Problem Formulation}

We consider the application scenario of multi-robot exploration as a partially observable distributed environment. Assuming that each agent could obtain the accurate positions of other agents and the obstacles within its visual range. Each agent $i$ learns a policy $a_i = \pi_i(o_1, o_2,..., o_N)$ on $N$ observations which maps each agent's observation $o_i$ to a distribution over the actions $a_i$. The learning process of the individual policy has to regard observations from other agents with focus, so that the number of involving agents could be extended. Therefore, each agent should consider other agents' different contribution to the decision making process, rather than considering them all at all the time. Due to each agent could not weight other agents' observation on their own, they should learn the ability to decide the importance of shared information and calculate a integrated contribution.

\subsection{Framework}

More formally, multi-head attention mechanism is introduced to centrally learn a critic to enable each agent to select which agents to attend to at each time step. The shared critic receives the observations, actions and historical trajectories from all agents and generates Q-values for each agent, and the contribution of other agents' information is evaluated by multiple attention heads through attention weight. In the training process, all critics are updated together by minimizing a joint loss function. The main architecture of MAMECS is shown in Fig. 1. 

\begin{figure}[htbp]
	\centering
	\includegraphics[width=0.8\textwidth]{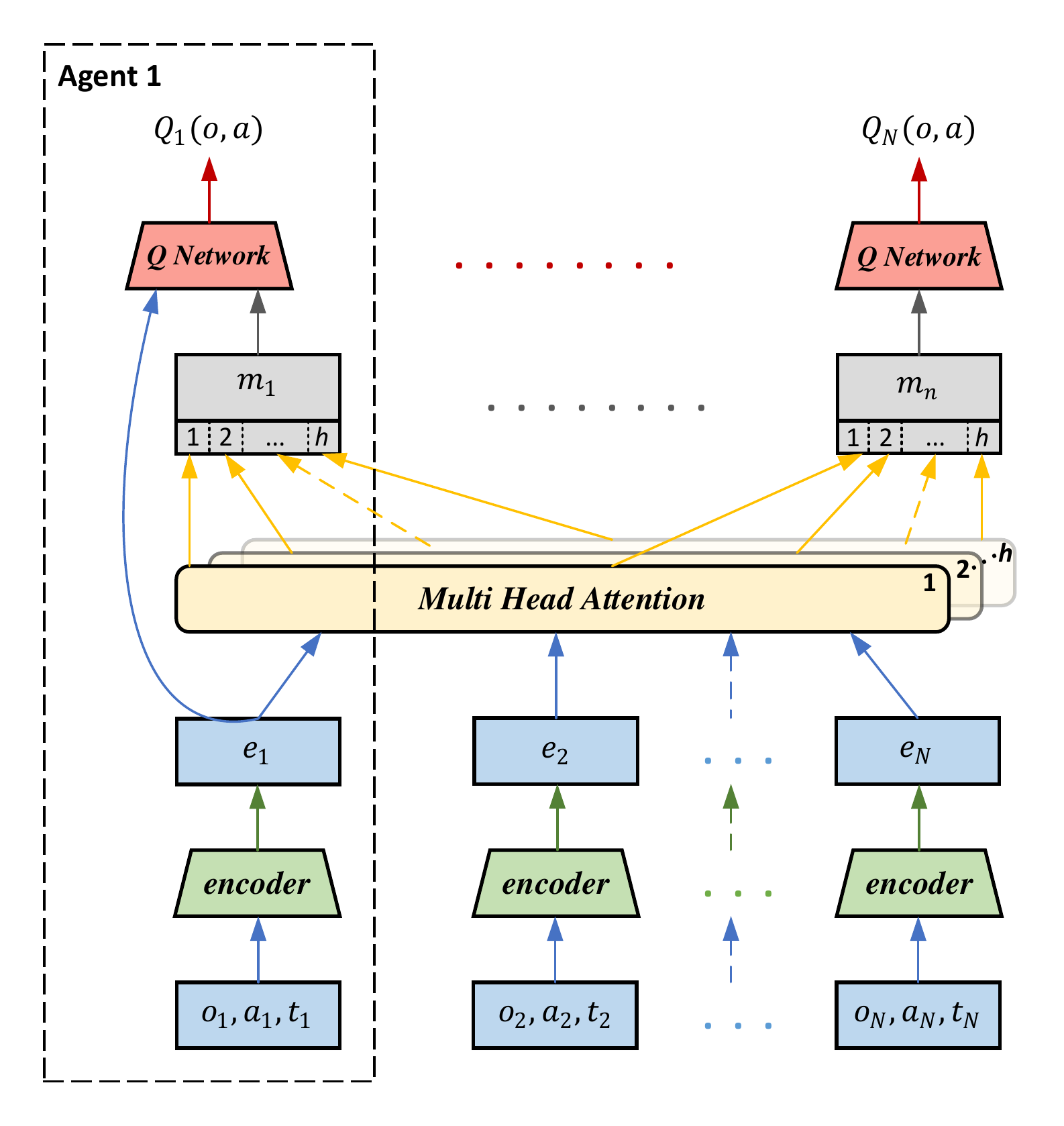}
	\caption{The architecture of MAMECS for $N$ agents. The multi-head attention determines the attention weight between each agent based on the inputs of agents' observations, actions and historical trajectories.}
\end{figure}

The goal of our method is to selectively paying attention to other agents’ information in agent's decision-making process. In detail, the encoder takes each agent's observation, action and historical trajectories $o_{i}, a_{i}, t_{i}$ as input, then the encoded information $e_{i}$ are fed into the multi attention heads to generate the integrated message. $e_{i}$ is evenly separated into $h$ parts, equaling the number of attention heads. Each attention head employs a separate set of parameters to weight the different contribution from other agents, and each part of the information with different weights is concatenated as a single vector $m_{i}$. Each agent takes the relevant information of concatenated message $m_{i}$ and the local information $o_{i}, a_{i}, t_{i}$ into account for estimating its value function $Q_{i}^{\psi}(o,a)$.

The shared critic receives the observations $o={o_1, ..., o_N}$, actions $a={a_1, ..., a_N}$ and historical trajectories $t={t_1, ..., t_N}$ from all agents indexed by $i\in{1,..., N}$. The Q-value function $Q_{i}^{\psi}(o,a)$ for the agent $i$ can be calculated as
\begin{equation}
Q_{i}^{\psi}(o,a)=f_{i}(e_i, x_i)=f_{i}(g_{i}(o_i, a_i, t_{i}), x_i)
\end{equation}

Here, $f_i$ represents the Q-Network of a two-layer multi-layer perception (MLP) and $g_{i}$ is the encoder of a one-layer MLP embedding function. $x_i$ represents the contribution from all other agents, which is the concatenated vector from $\mathcal{H}$ attention heads. For attention head $h$, the corresponding part of $x_i$ is the weighted sum of the attention weight $\alpha_{ij}^{h}$ and the embedded information $v_{j}$. To be concrete, agent $j$'s embedded information $e_{j}$ is transformed by matrix $W_{v}^{h}$ into a value.

\begin{equation}
x_i=Concat_{h}(\sum_{j\ne{i}}\alpha_{ij}^{h}W_{v}^{h}e_{j}), \forall h\in \mathcal{H}
\end{equation}

To evaluate the corresponding attention weight $\alpha_{ij}^{h}$ of agent $j$ to the agent $i$, a bilinear mapping is used to project the embedded features into a query-key system. $W_q^h$ transforms agent $i$'s embedded information $e_i$ into a ``query" and $W_k^h$ transforms agent $j$'s embedded information $e_j$ into a ``key". We then perform Softmax operation to process the similarity value between these two embeddings.  

\begin{equation}
\alpha_{ij}^{h} = \frac{exp\bigg(W_{q}^{h}e_{i} \cdot (W_{k}^{h}e_{j})^{T}\bigg)}{\sum_{r\ne{j}} exp\bigg(W_{q}^{h}e_{i} \cdot (W_{k}^{h}e_{r})^{T}\bigg)}
\end{equation}

Each attention head $h$ uses a separate set of parameters $(W_{q}^h, W_k^h, W_v^h)$ to process the embedded information, and calculate the contribution from all other agents to the current agent. The aggregated message from each attention head is then simply concatenated into a single vector.

\subsection{Learning the Shared Attentive Critic}  	  
%%%%%%%%%%%%%%%%%%%%%%%%%%%%%%%%%%%%%%%%%%%%%%%%

As for the question of how to update all critics together within an shared attentive critic. Due to the critic parameters are shared across all agents, all critics can be updated together by minimizing a joint regression loss function:

\begin{equation}
L_Q(\psi)=\sum_{i=1}^N E_{(o,a,r,o^{'})\sim{D}}[(Q_i^{\psi}(o,a)-y_i)^2] \\	
\end{equation}
where
\begin{equation}
y_i = r_i+\gamma E_{\alpha^{'}\sim{\pi\bar{\theta}}(o^{'})}[Q_{i}^{\bar{\psi}}(o^{'},a^{'})-\alpha log(\pi_{\bar{\theta_i}}(a_i^{'}|o_i^{'}))]
\end{equation}
where $Q_i^{\psi}$ is the estimate action-value for agent $i$, while $y_{i}$ is the ground-truth value. $\bar{\psi}$ and $\bar{\theta}$ are respectively the parameters of target critics and target policies. $\alpha$ is used to balance the entropy and rewards. So each agent's policy is updated with the following gradient:

\begin{equation}
\begin{aligned}
\bigtriangledown_{\theta_i}J(\pi_\theta)&=E_{\alpha\sim\pi_\theta}[\bigtriangledown_{\theta_i}log(\pi_{\theta_{i}}(a_i|o_i))(\alpha log(\pi_{\theta_{i}}(a_i|o_i))\\
&-Q_i^{\psi}(o,a)+b(o,a_{\backslash i})]
\end{aligned}
\end{equation}

We represent the set of all agents except $i$ as $\backslash i$. $b(o,a_{\backslash i})$ is the multi-agent baseline, which is the average action value of all agents:

\begin{equation}
\begin{aligned}
b(o,a_{\backslash i})=E_{a_{i}\sim{\pi_{i}(o_{i})}}[Q_{i}^{\psi}(o, (a_{i}, a_{\backslash i}))]
=\sum_{a_{i}^{'}\in{A_{i}}}\pi(a_{i}^{'}|o_{i})Q_{i}(o,(a_{i}^{'},a_{\backslash i}))
\end{aligned}
\end{equation}

The baseline can assist each agent judge its own contribution to the team in a cooperation scenario. By comparing agent's Q-value with the average action value, the certain contribution of agent to the reward value can be found. Full training details and hyperparameters can be found in the following subsection. 

%%%%%%%%%%%%%%%%%%%%%%%%%%%%%%%%%%%%%%%%%%%%%%%% 

\subsection{Dynamic Environment Modeling} 

Rather than modeling the environment using occupying grids, we instead model the two-dimensional world in continuous space. Consequently, a robot can move to arbitrary positions on the map rather than only positions on the grids surrounding it , which represents a more flexible and practical approach to the multi-robot exploration task. The basic idea behind our approach is to represent each agent as the center point of a circle, so that the agent can explore the area within the radius of this circle at each time step.

We make some definitions first: $o_{i}^t$ stands for the observation of robot $i$ at time step $t$, which is indexed by $i\in{1,..., N}$;  $a_{i}^t$ stands for the output actions given the corresponding inputs; $x_{i}^t$ is the contribution from other agents, a weighted sum of the attention weight and the embedded information of other agents. $t_{i}^t$ stands for the coordinates of agent's history trajectory. $\pi$($a_{i}^t|o_i^t,t_{i}^t,x_{i}^t)$
stands for the policy of choosing controls based on the past observations, trajectories and contribution from other agents.

We assume that there is an underlying map $M = f(S_{tra})$ primarily unknown to the agents. To be concrete, $S_{tra} = \left\{ i, t, x_i, y_i\right\}$ is a dynamic coordinate set which records the position of robot $i$ at time step $t$. Each robot wishes to infer its belief map over map $M$ at time $t$ given all its previous  observations, trajectories and other robots' weighted contribution leading up to that time step. To simplify the problem, we assume the individual map for each agent, indexed as $m_{i}$, are independent:

\begin{equation}
b_{t}(m)=\prod_{i}p(m_{i}|o_{1:t},t_{1:t},x_{1:t})=\prod_{i}b_{t}(m_{i})
\end{equation}

In information gain approaches, the goal of exploration is twofold - not just to map the environment but to move the robot to maximize the amount of new information in the environment. We apply an information gain method to measure the environment uncertainty in a probability distribution $b_{t}(m)$ by the entropy $H(b_{t}(m))$

\begin{equation}
H(b_{t}(m))=\int H(b_{t}(m_{i}))\, dm_{i} = \int b_{t}(m_{i})\log b_{t}(m_{i})\, dm_{i}
\end{equation}

This is a measure of the uncertainty associated with the constructed belief map $m_i$. As $b_{t}(m)$ becomes more peaked $H(b_{t}(m))$ decreases, and $H(b_{t}(m))$ reaches zero when the outcome of a random trial is certain.  	
%%%%%%%%%%%%%%%%%%%%%%%%%%%%%%%%%%%%%%%%%%%%%%%%

\subsection{Entropy-oriented Reward Function}

Now, we describe our reward functions which encourage the agents to explore more unknown dynamic environment in the shortest time. In the learning process, there is a central node that records the trajectory of each agent and gives the corresponding reward based on their performances. At the time step $t$, the agent obtains its own observation $o_{i}^t$ and the contribution from other agents $x_{i}^t$. The agent is likely to execute the action with highest reward $a_{i}^t$ and updates its belief map $b_t(m_i)$ based on the obtained inputs and its history trajectory $t_{i}^t$. 

To describe the reward function accurately, moreover, we first illustrate our expectations of the agents in the exploration tasks. Each agent is expected to avoid collisions with other agents and obstacles in the environment, avoid exploring the same area repeatedly, explore the map in minimal time steps, and reduce the uncertainty for the whole map as soon as possible. In other words, the tasks we encourage agents to do are rewarded positively, while behavior we wish the agents to avoid is rewarded negatively . So at the time step $t$, each agent seeks a policy $\pi(a|o,t,x)$ that could reach the expected goals. Reward function $R_{t}$ is as follows:

\begin{equation}
R_{t} = H(b_t)-H(b_{t+1})+ C^tr_{coll}+F(S_{inter}^t)
\end{equation}

Here, $R_{t}$ is the combination of three aspects: the environment's reduced entropy, collision information, and the repeated area coverage information. $H(b_t)-H(b_{t+1})$ is the information gain after agent taking an action, which is defined to be the decrease in entropy. In the context of robotic exploration, we measure the information gain with the difference value of the map entropy between time step $t$ and $t+1$. It is the value that we wish to maximize by selecting new poses. As for the collisions with other agents and obstacles, $C^t$ refers to the number of collisions. Two agents collide if the circle centered on their coordinates coincides. A collision incurs a negative reward $r_{coll} = -10$. 

$F(S_{inter}^t)$ is applied to calculate the repeated area coverage information, which is a piecewise function of agents' intersection area $S_{inter}$. We assume that the environment is a square while the agent is represented as a circle. For the problem of fully cover a square with a minimum amount of radius circles, there is no known way to find optimal solutions. However, in our case, agents are not required to fully cover the environment but are expected to achieve the maximum coverage ratio. So, we propose a theorem on coverage ratio and design a calculate method to achieve this goal (shown in Appendix A). 

We have proved that there is a better arrangement of circles to achieve a higher cover ratio. As a result, the piecewise function $F(S_{inter})$ reaches its maximum when $S_{inter}$ equals the intersection area of two circles in the second circumstance, where there is a higher coverage ratio (the second circumstance shown in  Fig. A10 in Appendix A). $F(S_{inter})$ is designed to be a continuous function, and agents could receive reward signals in the whole exploration process. So, the sparse reward problem can be avoided, agents are able to learn the exploration task better. 

\begin{equation}
\begin{split}
F(S_{inter})=\left\{
\begin{aligned}
&\frac{r_{2}-r_{1}}{S_{2}}\cdot S_{inter} + r_{1} , {0 \leq S_{inter}< S_{2}}\\
&\frac{r_{2}-r_{3}}{(S_{2}-S{3})^2}\cdot (S_{inter}-S{3})^2+r_{3} , {S_{2} \leq S_{inter} \leq S_{3}}
\end{aligned} \right.
\end{split}
\end{equation}

where, 

\begin{equation}
\left\{
\begin{aligned}
&S_{2} =  4\cdot S_{PQB} = (\pi-2) r^2 \\		
&S_{3} = S_{circle} = \pi r^2
\end{aligned}		
\right.
\end{equation}

\begin{figure}[!t]
	\centering
	\includegraphics[width=0.7\textwidth]{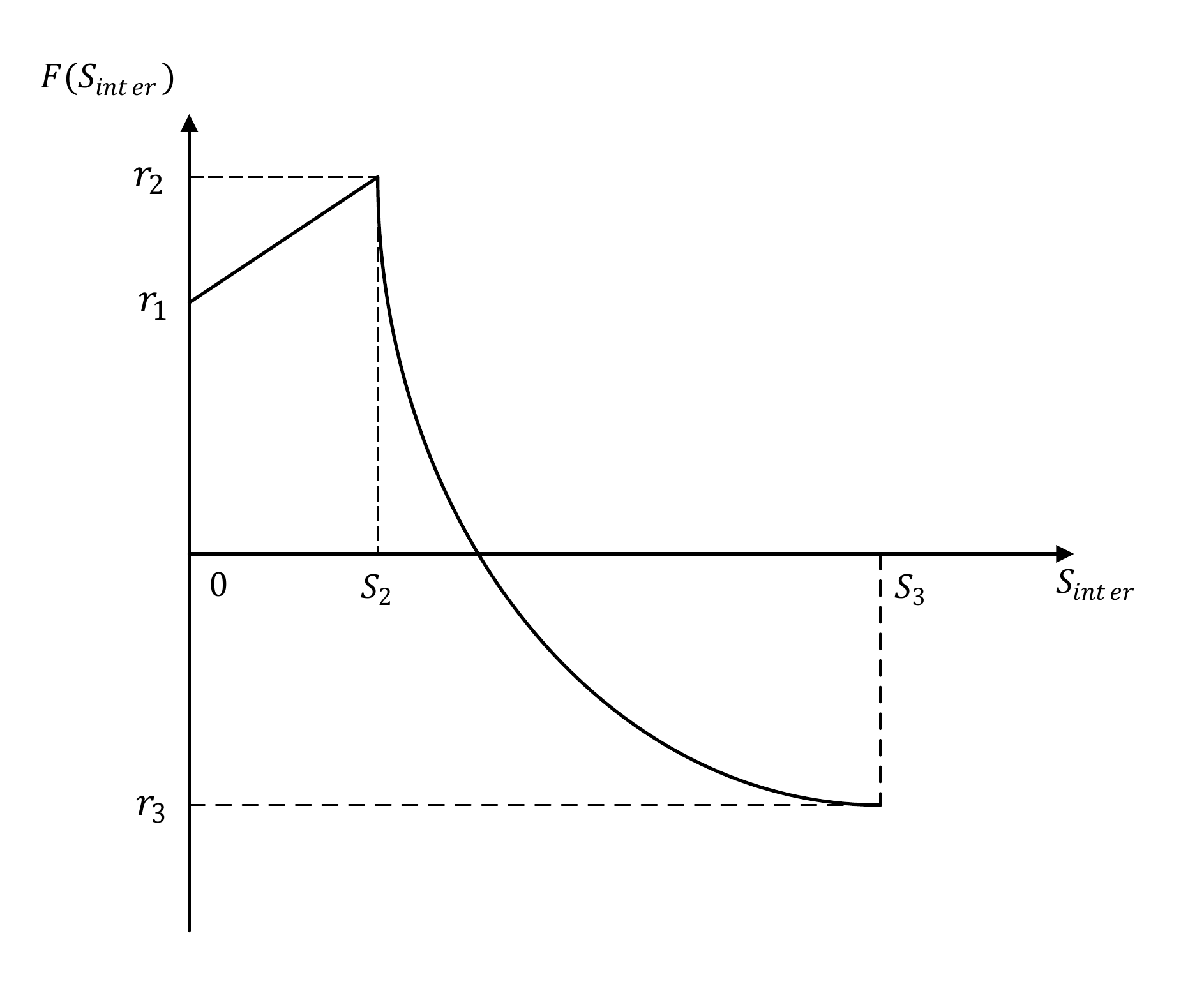}	
	\caption{The piecewise function $F(S_{inter})$ that calculates the repeated area coverage information based on the agents' intersection area $S_{inter}$.}
	\label{fig.2}
\end{figure}

When circles tangent to each other, the coverage ratio may not lead to the maximum while the agent explores new areas in every time step, so the reward $r_{1}=2$. With the increase of two circles' intersection area, the value of reward increases and reaches its maximum $r_{2}=3$ at $S_{inter} = S_{2}$. Then $F(S_{inter})$ falls in the form of a quadratic function and reaches bottom $r_{3}=-3$ when two circles coincide completely.

\subsection{Exploration Rate-based Training Approach}

To simulate the dynamic obstacles in the practical environment and to enhance the robustness of agents to dynamic settings, we gradually add $n$ random obstacles every $m$ time-steps to the original environment. Thus, each agent is expected to learn the strategy to avoid both static and dynamic obstacles. To ensure the model find better local optimum and accelerate the training speed, curriculum learning is adapted to the training process by gradually increasing task difficulty. In detail, the value of $m$ decreases during the training time so that the frequency of adding random obstacles increases, which means the difficulty of the mission is increasing. However, the value of $\frac{n}{m}$ keeps constant, which means the number of random obstacles $n$ also decreases when they are added more frequently. It is an essential setting due to the value of $el_{rate}$ (a crucial component to measure the success standard) is fixed in the given time-steps. 

Each simulation is terminated after a specified number of time-steps and classified as a failure if collisions with obstacles have occurred or the exploration rate $el_{rate}$ is less than $90\%$. $el_{rate}$ here is calculated as follows:

\begin{equation}
el_{rate}=\frac{S_{explored}\bigcup S_{obstacles}}{S_{map}}
\end{equation}

Here, $S_{explored}$ is the explored area in the map, which is the union area of agents' trajectory in each time step. $S_{obstacles}$ is the subset of final obstacles, including the static and random obstacles in the map. $S_{map}$ means the total area of the environment and is modeled as a square in this scenario. Since each agent and obstacle is represented as a circle with a certain radius and the positions of newly generated circles may overlap with the area that has been explored by the multi-agent system, we take the operation $\bigcup$ to calculate the union area of these circles. 

We use Simpson adaptive algorithm, a classic computational geometry method to calculate the union area of the circles. We first judge the position between circles to optimize calculation. If the centers of the circles coincide, then only the area of one circle is retained; Or if the distance from the center of a circle to any other center exceeds the radius, then we add the area of a complete circle to the total area. After such screening, we use the Simpson adaptive algorithm to calculate the area corresponding to each arc. We first randomly segment an arc, for each interval $\left[l,r\right]$, we recursively calculate the values corresponding to the endpoints and the intermediate point $f(l), f(r), f(mid)$. $f(i)$ is taken as the total length of the transversal lines of $x=i$ and all the circles, so the area between the interval (l, r) is:

\begin{equation}
\int_{l}^{r}f(x) = \frac{r-l}{6}(f(l)+4f(mid)+f(r))
\end{equation}

%%%%%%%%%%%%%%%%%%%%%%%%%%%%%%%%%%%%%%%%%%%%%%%

\section{Experiments}
\label{sec:experiments}

In this section, we will first introduce our experimental settings and locations storage. Then, we will show the training performance compared with the baseline methods. Finally, we will give the attention visualization and the corresponding analyze.

\subsection{Experiment Set Up}

We use MPE (Multi-agent Particle Environment) framework to construct an environment to test various capabilities of our approach (MAMECS) and baselines. The square map of size $1\times1$ represents an artificial environment with various obstacles, which can satisfy the amount of exploration needed to test our method, but not too large to cause inadequate exploration. The experimental environment has continuous action space, so the agent can move to anywhere on the map determined by its velocity and acceleration parameters. Each agent can sense the environment information within the exploration radius of $r=0.04$ and has a communication range covering the whole environment. The goal for the whole system is to explore the map as much as possible in a fixed time.

To be concrete, four agents enter the environment through four arrival points and the positions traveled by each agent form a trajectory, represented as red circles within the same radius of agent $r=0.04$. As for the obstacles, there are 4 original blocks in the prime environment and new blocks are introduced according to a uniform random distribution across the search space. The size of the obstacles is the same as that of the agent and the obstacles will stay on the map until the end of the episode. For the number of agents, two new agents enter the map randomly from the four arrival points every 4 time-steps. However, the total number of agents at a given time is limited to $N_{max}=16$. Each agent has a life cycle of 60 time-steps and is encouraged not to collide with other agents and obstacles as well as to keep inside the map.  

\begin{figure}[!t]
	\centering
	\includegraphics[width=4.3in]{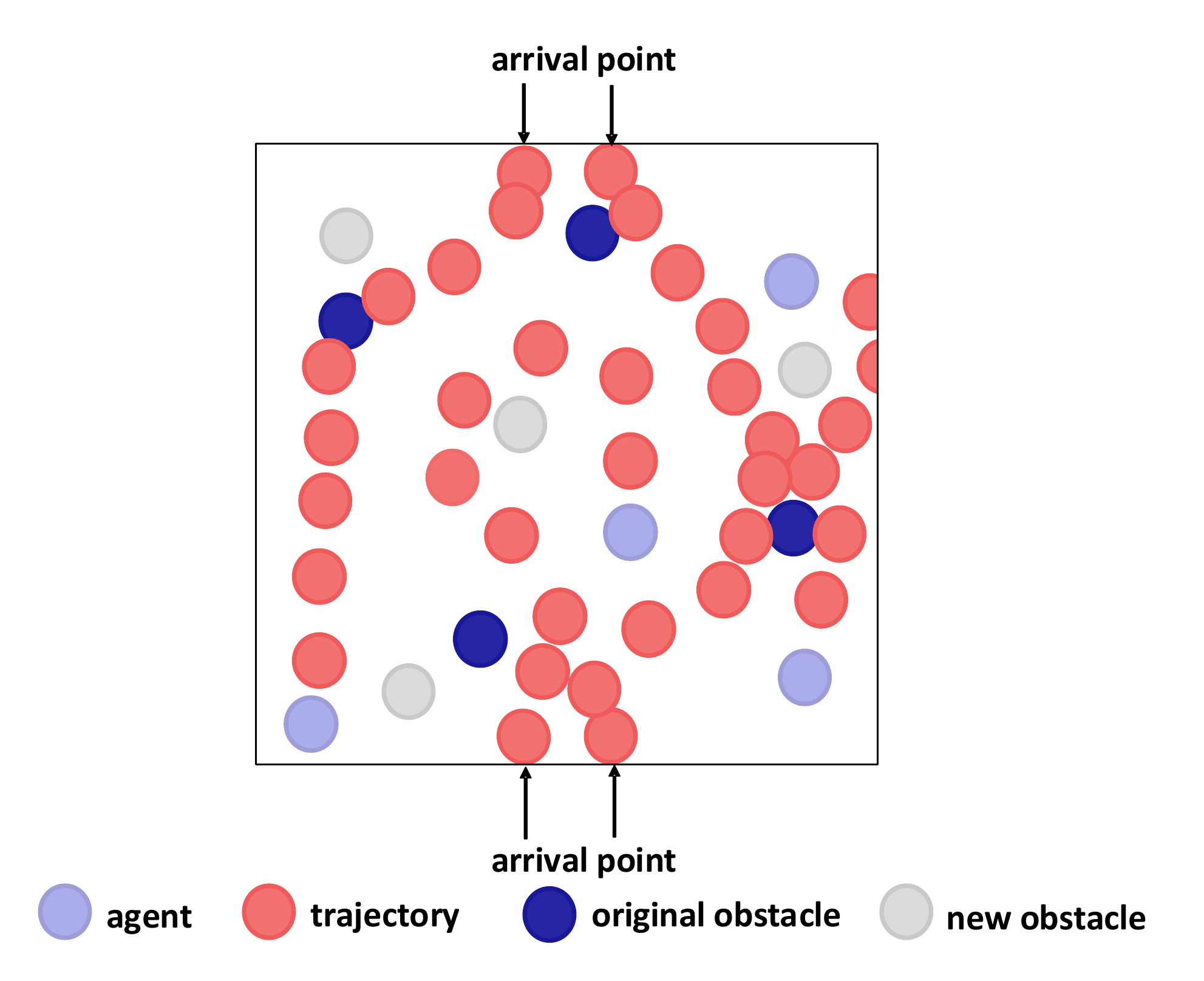}
	\caption{The experimental environment in continuous space which has dynamic number of blocks.}
	\label{fig.3}
\end{figure}

%%%%%%%%%%%%%%%%%%%%%%%%%%%%%%%%%%%%%%%%%%  

\subsection{Storage of Location}

Agents' location information is a set of points in two-dimensional space, so we build a 2d tree to record and process these coordinates. A K-D tree is a space partitioning data structure for organizing points in a K-Dimensional space, k is 2 in our two-dimensional environment. Each leaf node in the binary tree is a 2-dimensional point and every non-leaf node can be thought of generating a splitting hyperplane that divides the space into two parts. Points to the left of this hyperplane are represented by the left subtree of that node and points to the right of the hyperplane are represented by the right subtree. 

We first use the initial agents' location in the environment to construct a balanced 2d tree. The feature with the largest variance is selected as the segmentation feature, so the segmented data will be relatively scattered. Then, we select the median of the feature as the segmentation point, thus the number of nodes in the left subtree and right subtree is approximately the same, which is convenient for binary search. When the agent moves to a new position, the coordinate will be recorded and added to the 2d tree. As for adding elements, we traverse the tree from the root node and move to either the left or the right child depending on which side of the node's splitting plane contains the new node. When the agent leaves the environment due to the battery problem, its position will also be removed from the 2d tree. 

To calculate the intersection area with different agents, we need to obtain the coordinates of its surrounding agents. This question can be thought of as the range search problem in the 2d tree. To find all points contained in a given query rectangle, which is centered on the coordinate of the current agent and the diameter of an agent's exploration range is used as a side length. We start at the root and recursively search for points in both subtrees using the following pruning rule: if the query rectangle does not intersect the rectangle corresponding to a node, there is no need to explore that node (or its subtrees). That is, search a subtree only if it might contain a point contained in the query rectangle.

\subsection{Training Performance}

As for our training procedure, we use an off-policy, actor-critic method Soft Actor-Critic for maximum entropy reinforcement learning in the training progress of 40000 episodes. There are 12 threads to process training data in parallel and a replay buffer to store experience tuples of $(o_{t}, a_{t}, r_{t}, o{t+1})_{1...N}$ for each time step. The environment gets reset every episode of 60 steps. The policy network and the attention critic network get updated 4 times after the first episode. In detail, we sample 1024 tuples from the replay buffer and update the parameters of the Q-function loss and the policy objective through policy gradients. Adam optimizer is used and the learning rate is set as 0.001. We use a discount factor $\gamma$ of 0.99 and 0.2 as our temperature setting for Soft Actor-Critic. The embedded information uses a hidden dimension of 128, and 4 attention heads are used in our attention critics. 

We compare our method MAMECS to two recently proposed approaches: MADDPG \cite{lowe2017multi} and COMA \cite{foerster2018counterfactual}, in the exploration task for each agent. MADDPG extends the traditional actor-critic methods for multi-agent mixed cooperative-competitive environments and becomes a common baseline method in various multi-agent scenarios. Unlike MADDPG, COMA uses a centralized critic to estimate the Q-function and decentralized actors to optimize the agents' policies. All methods have approximately the same number of parameters across agents, and each model is trained with 6 random seeds each. Hyperparameters for each underlying algorithm are tuned based on performance and kept constant across all variants of critic architectures for that algorithm.

\begin{figure}[!t]
	\centering
	\includegraphics[width=4in]{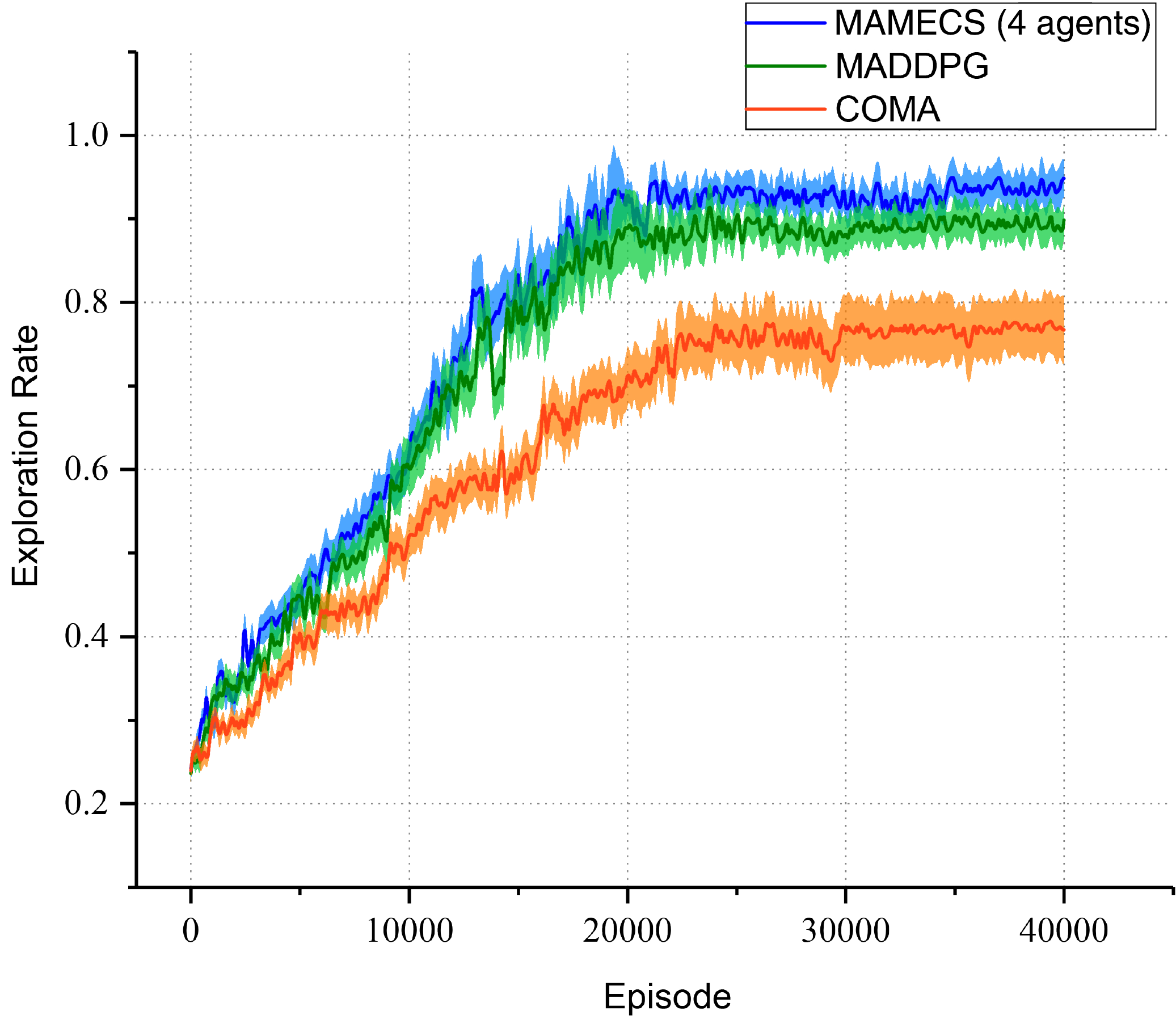}
	\caption{Exploration rate of 4 agents on MAMECS and baselines.}
	\label{fig.3}
\end{figure}

The performance of each approach is assessed by the average exploration rate in each episode. As shown in Fig. 4, MAMECS outperforms MADDPG AND COMA in the exploration rate and respectively reaches 94.65\%, 91.52\%, and 77.78\%. This indicates that MAMECS has a better learning ability in the exploration task, which contributes to the capability of focusing other agents' relevant information determined by the attention heads. To be concrete, although MADDPG takes other agents' observations as input, MADDPG does not weight the information differently.  COMA uses a single centralized critic network for all agents which may perform best in environments with global rewards and agents with similar action spaces. However, our environments have agents facing completely independent situations of different rewards. 

\begin{figure}[!t]
	\centering
	\includegraphics[width=4in]{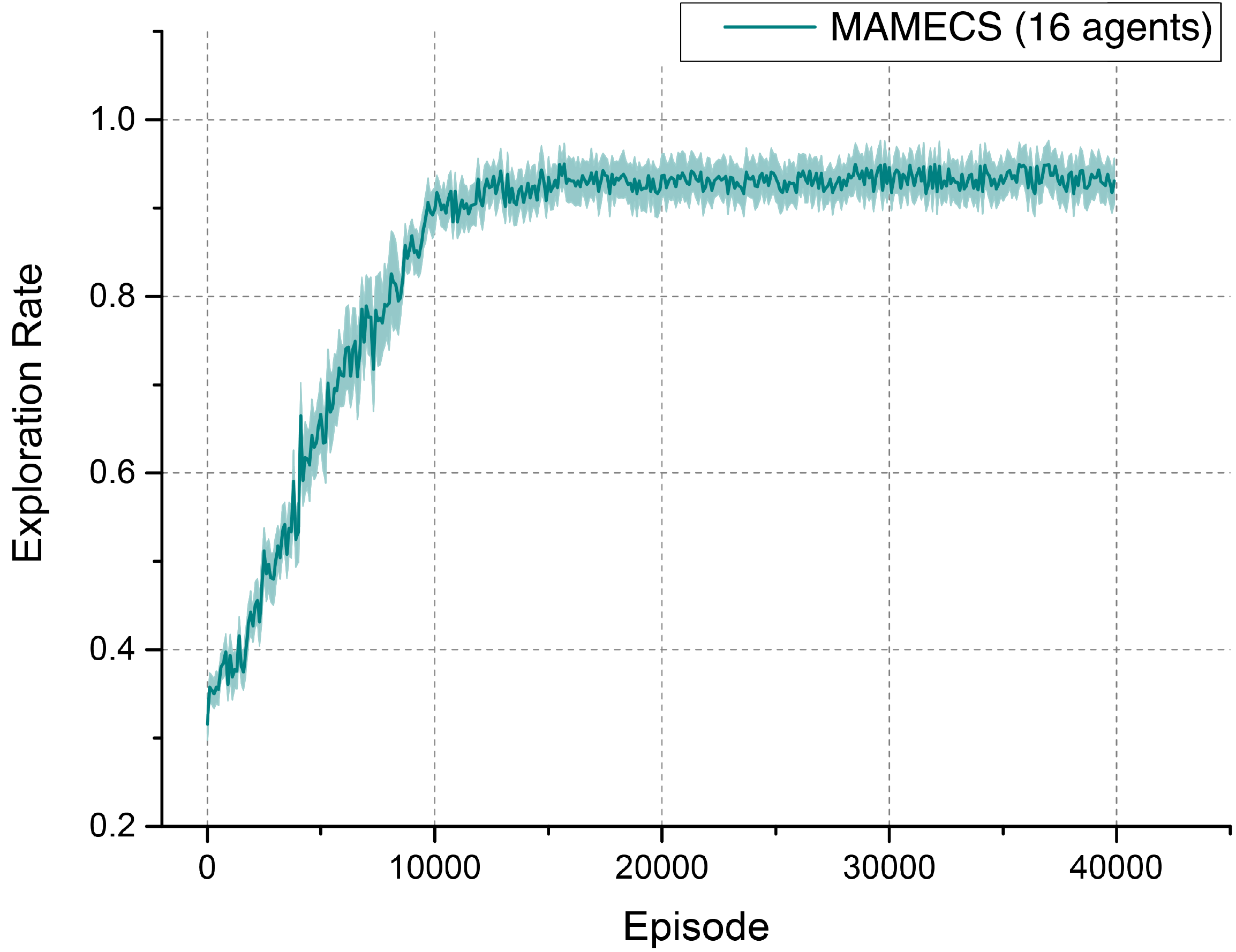}
	\caption{Exploration rate of 16 agents on MAMECS. Error bar are a 95\% confidence interval across 6 runs.}
	\label{fig.3}
\end{figure}

\begin{figure*}[hbtp]
	\centering
	\subfigure[]{
		\begin{minipage}{2.55in}
			\centering
			\includegraphics[width=2.55in,height=1.9in]{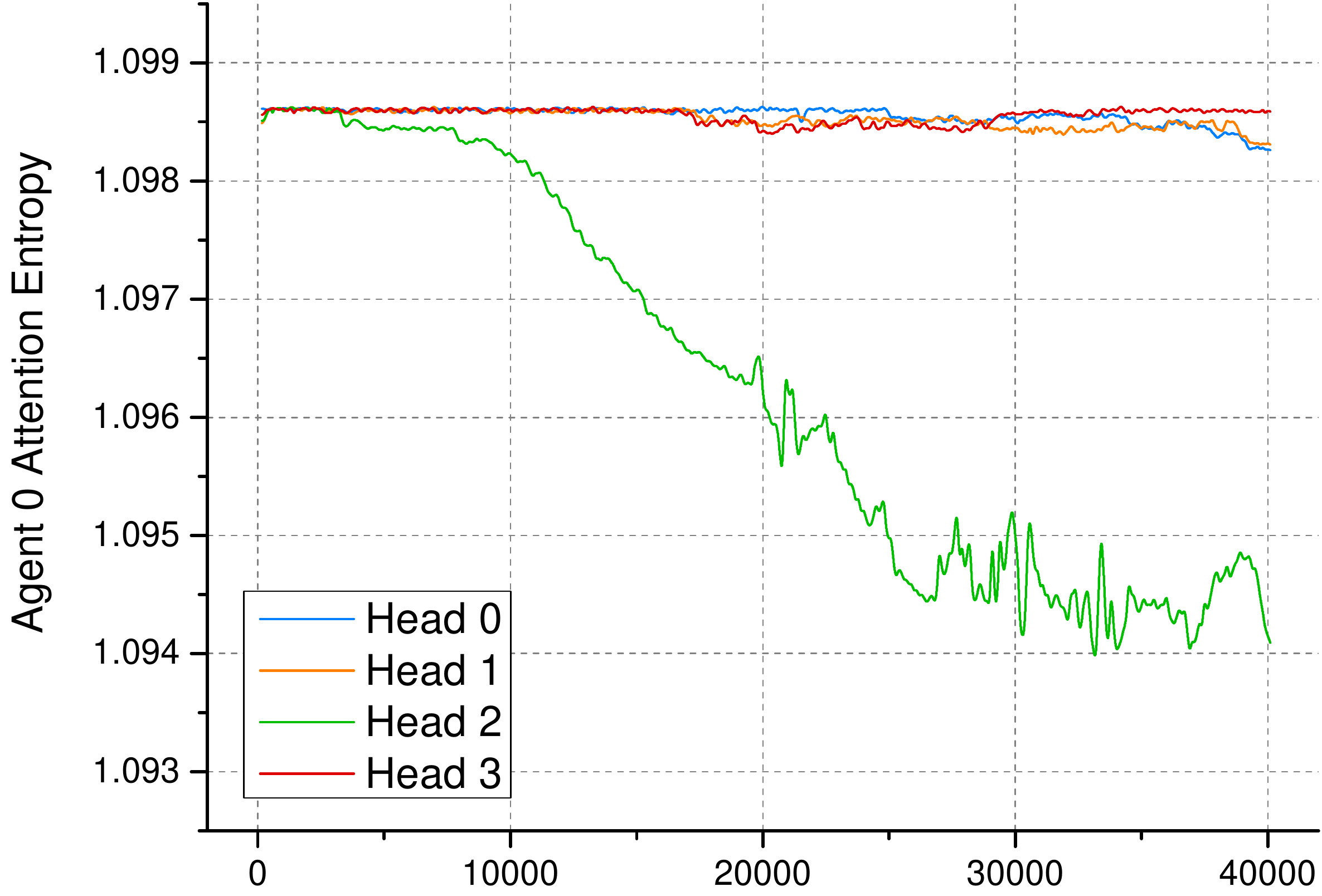}	
		\end{minipage}
	}
	\subfigure[]{
		\begin{minipage}{2.55in}
			\centering
			\includegraphics[width=2.55in,height=1.9in]{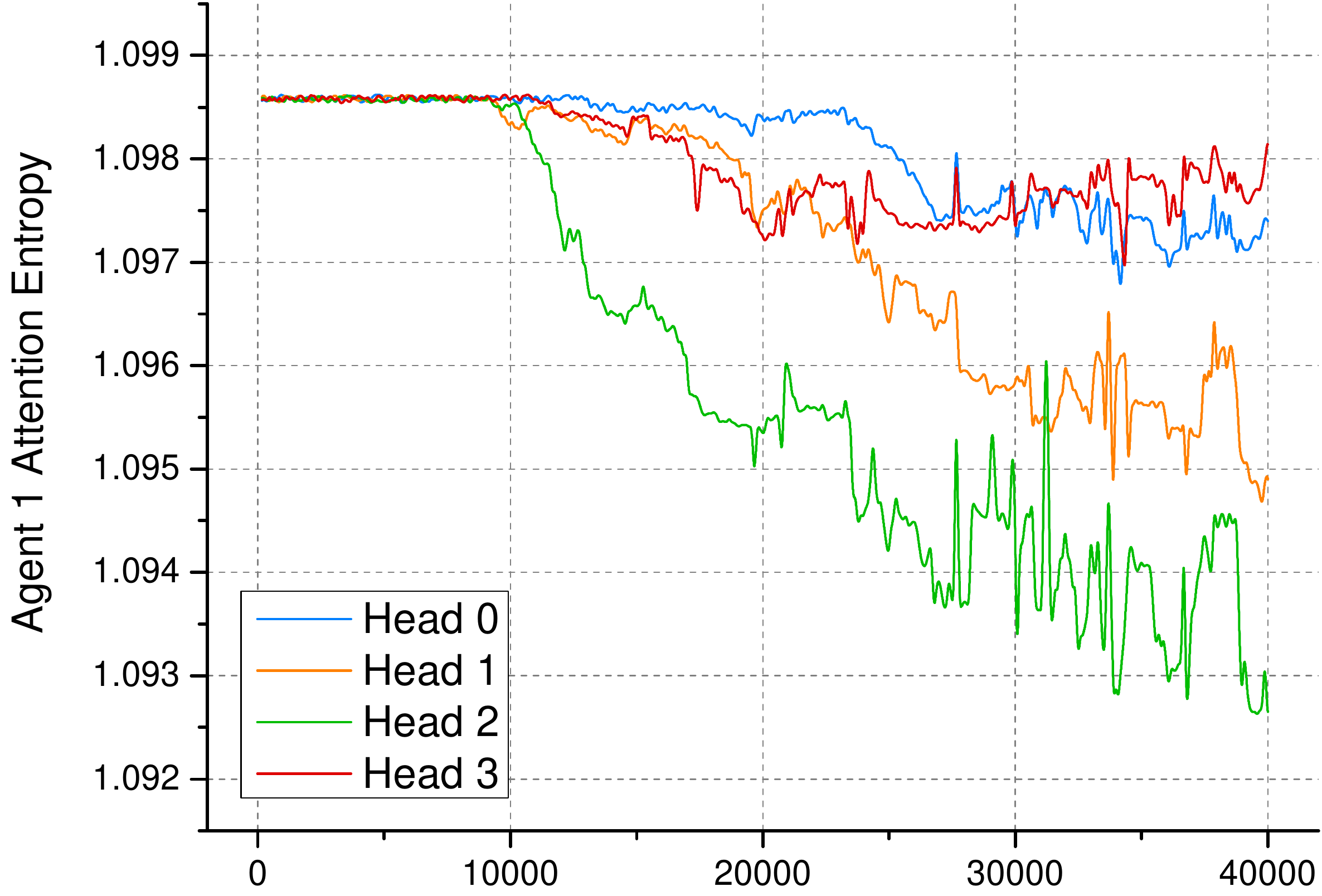}	
		\end{minipage}
	}
	\subfigure[]{
		\begin{minipage}{2.55in}
			\centering
			\includegraphics[width=2.55in,height=1.9in]{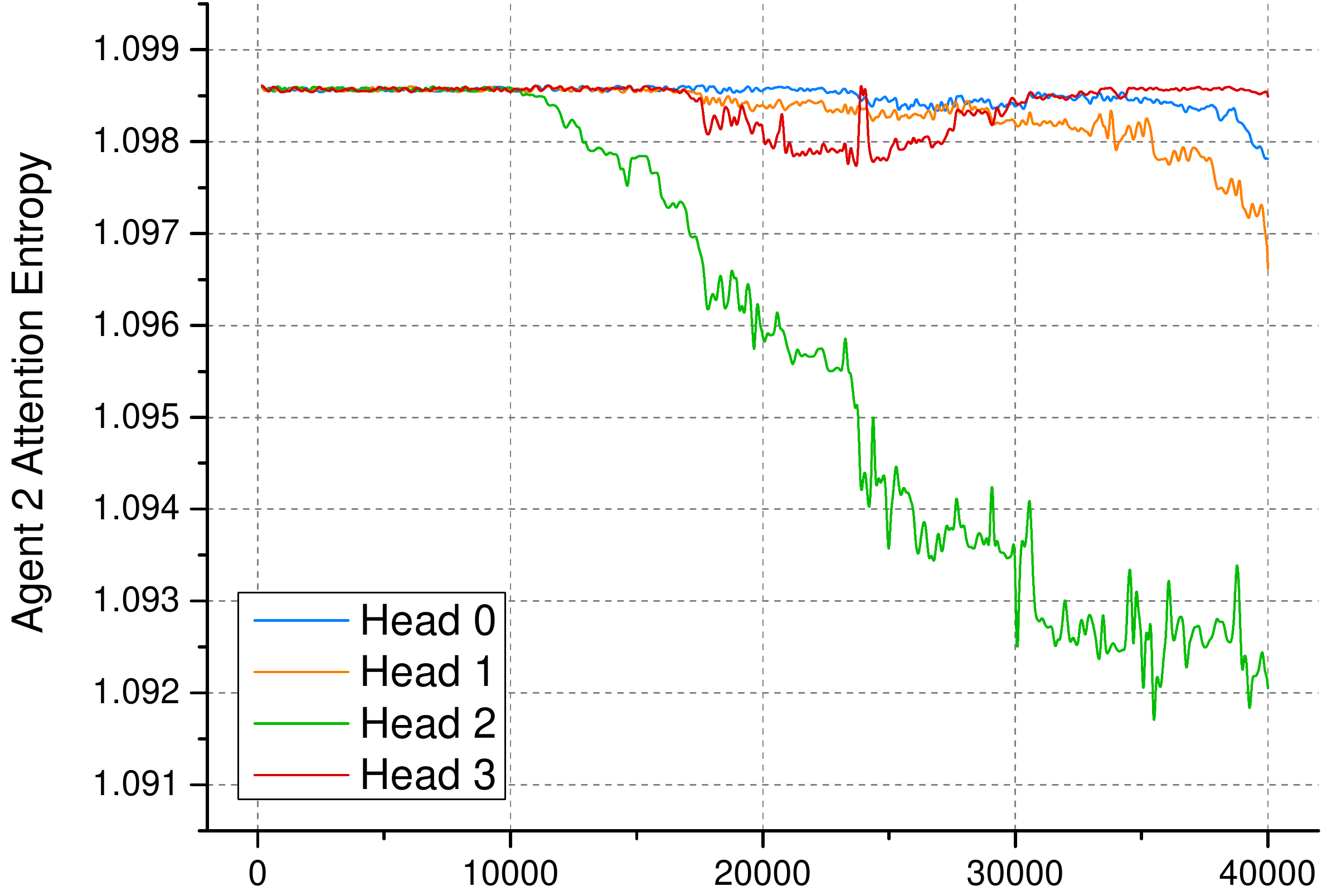}	
		\end{minipage}
	}
	\subfigure[]{
		\begin{minipage}{2.55in}
			\centering
			\includegraphics[width=2.55in,height=1.9in]{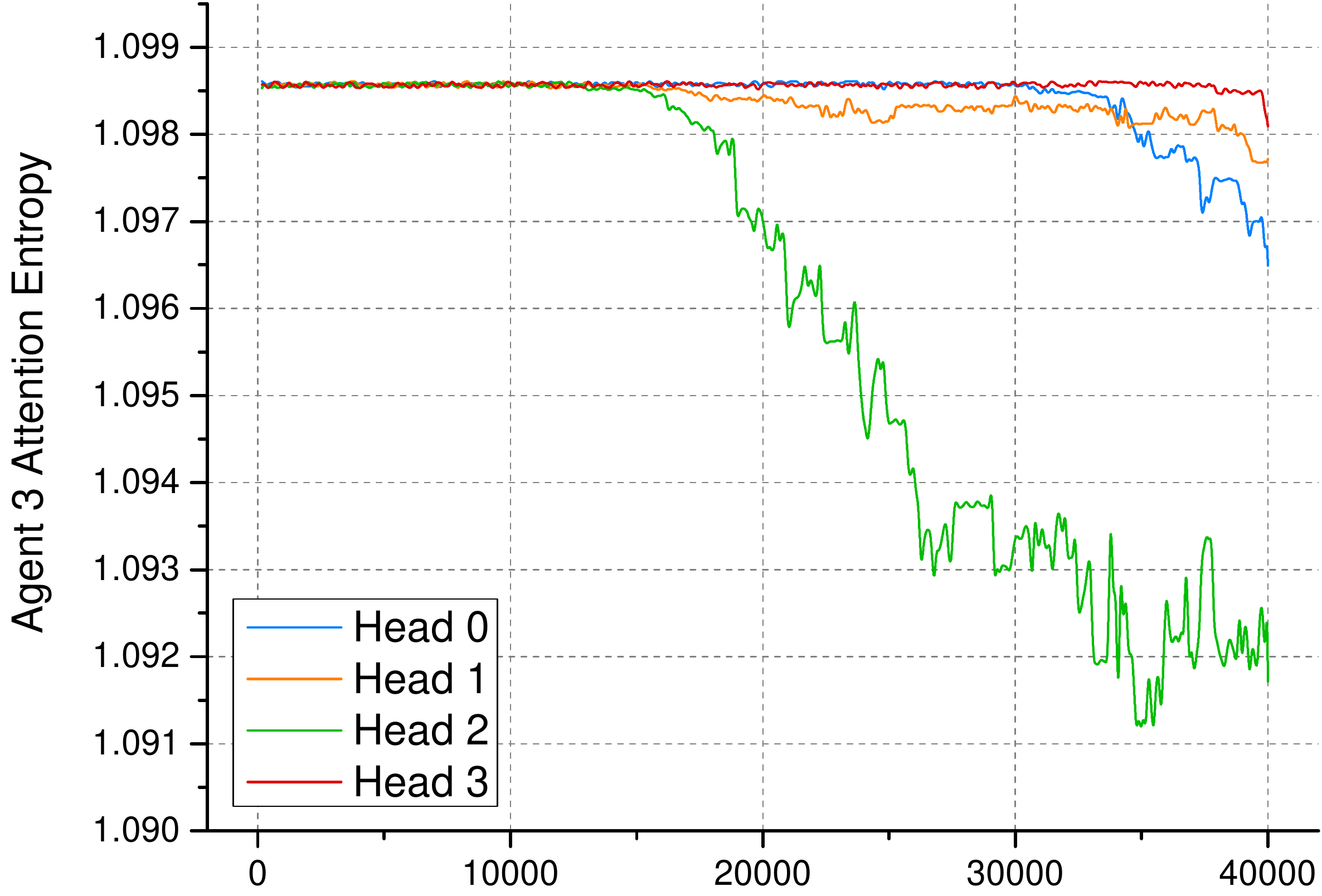}	
		\end{minipage}
	}	
	\caption{Attention ``entropy'' for each head over the course of training for the four agents in the multi-robot exploration environment}	
\end{figure*}

Due to the action space size increasing exponentially with the number of agents in MADDPG and COMA, the exploration task for 16 agents is not trainable. However, MAMECS only focus on the relevant information from other agents, which is equivalent to pruning the space to linearly increasing with the number of agents. Thus, the exploration task could extend to 16 agents through our approach. Meanwhile, exploration rate converges faster when more agents get involved (shown in Fig. 5). 

After training, we evaluate MAMECS, MADDPG, and COMA by running 1000 episodes and compare the number of collisions, the exploration rate and the average rewards at the end of each episode. As shown in Table 1, MAMECS outperforms other methods in all aspects. MAMECS not only increase the exploration ratio by 2.83\% than MADDPG but also reduce the collisions during the exploration process. Meanwhile, MAMECS has a higher reward which means better performance in the exploration task.

\begin{table}[htbp] 
	\centering
	\caption{\label{tab:test}Average performance over 1000 out-of-sample episodes in 60 time-steps.} 
	\begin{tabular}{lccc} 
		\toprule 
		\textbf{Approach} & \textbf{Collisions} & \textbf{Exploration-Rate} ($\%$) & \textbf{Average-Rewards}\\ 
		\midrule 
		MAMECS & \textbf{29 $\pm$ 12} & \textbf{93.25 $\pm$ 2.47}  & \textbf{63.61 $\pm$ 4.92}\\
		MADDPG & 37 $\pm$ 18 & 90.42 $\pm$ 2.53 & 55.21 $\pm$ 4.79\\ 
		COMA & 71 $\pm$ 16 & 76.64 $\pm$ 3.17 & 41.93 $\pm$ 5.87\\
		\bottomrule 
	\end{tabular} 
\end{table}
%\end{center}

%%%%%%%%%%%%%%%%%%%%%%%%%%%%%%%%%%%%%%%%%%

\subsection{Visualizing Attention}

Furthermore, in order to demonstrate the effect of the attention head on the agent during the training process, we test the “entropy” of the attention weights for each agent for each of the four attention heads that we use in the exploration task (Figures 6 and 7). A lower entropy value indicates that the head is focusing on specific agents, with an entropy of 0 indicating attention focused on one agent. In the exploration task for agents 0, 1, 2 and 3, we plot the attention entropy for each agent. In more detail, each agent tends to use a different combination of these four heads, indicating that each agent uses more than one  attention head in the exploration process, although their use is not mutually exclusive. This different combination of attention heads is appropriate due to the nature of the exploration task.

\begin{figure*}[hbtp]
	\centering
	\subfigure[]{
		\begin{minipage}{2.55in}
			\centering
			\includegraphics[width=2.55in,height=1.9in]{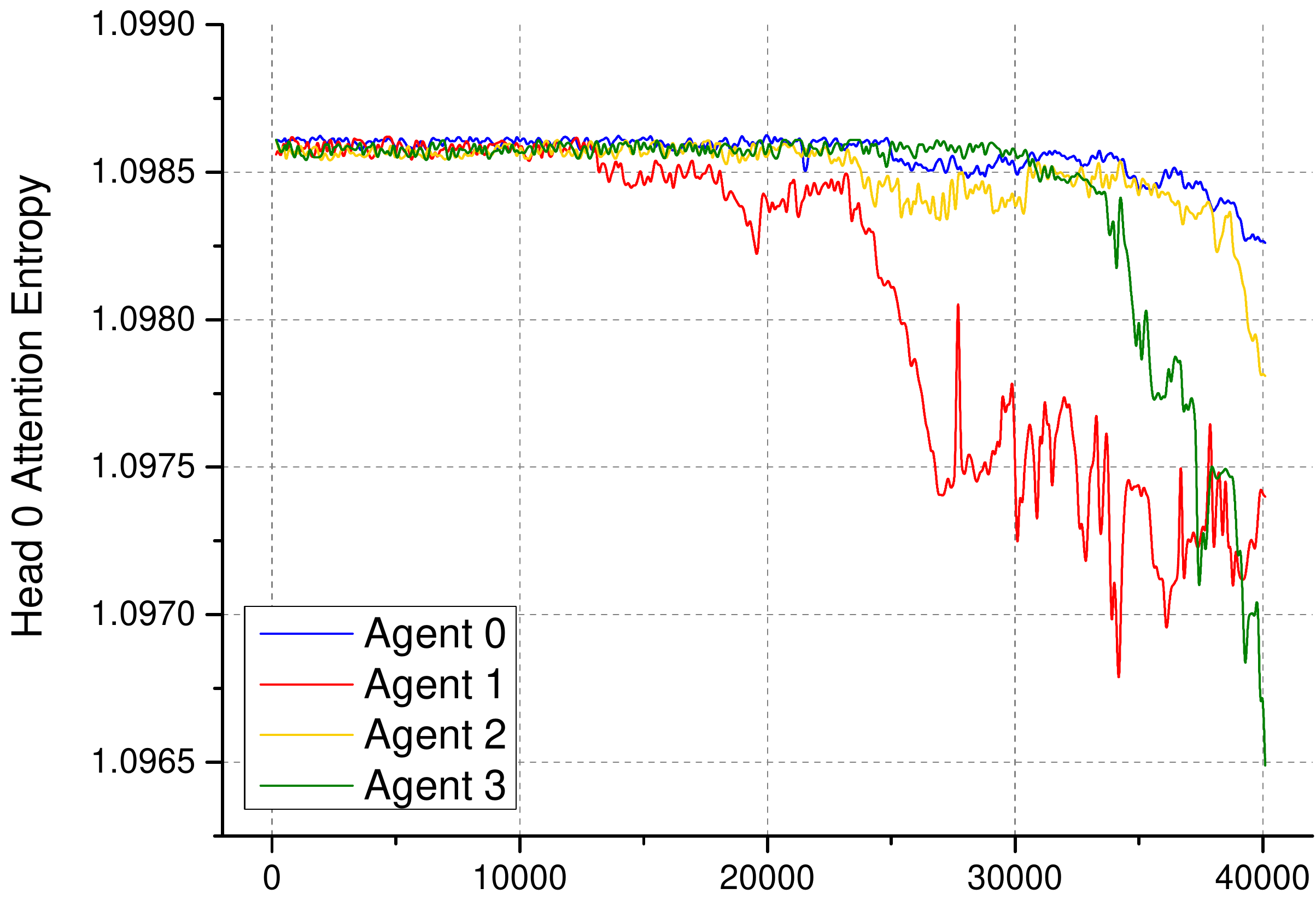}	
		\end{minipage}
	}
	\subfigure[]{
		\begin{minipage}{2.55in}
			\centering
			\includegraphics[width=2.55in,height=1.9in]{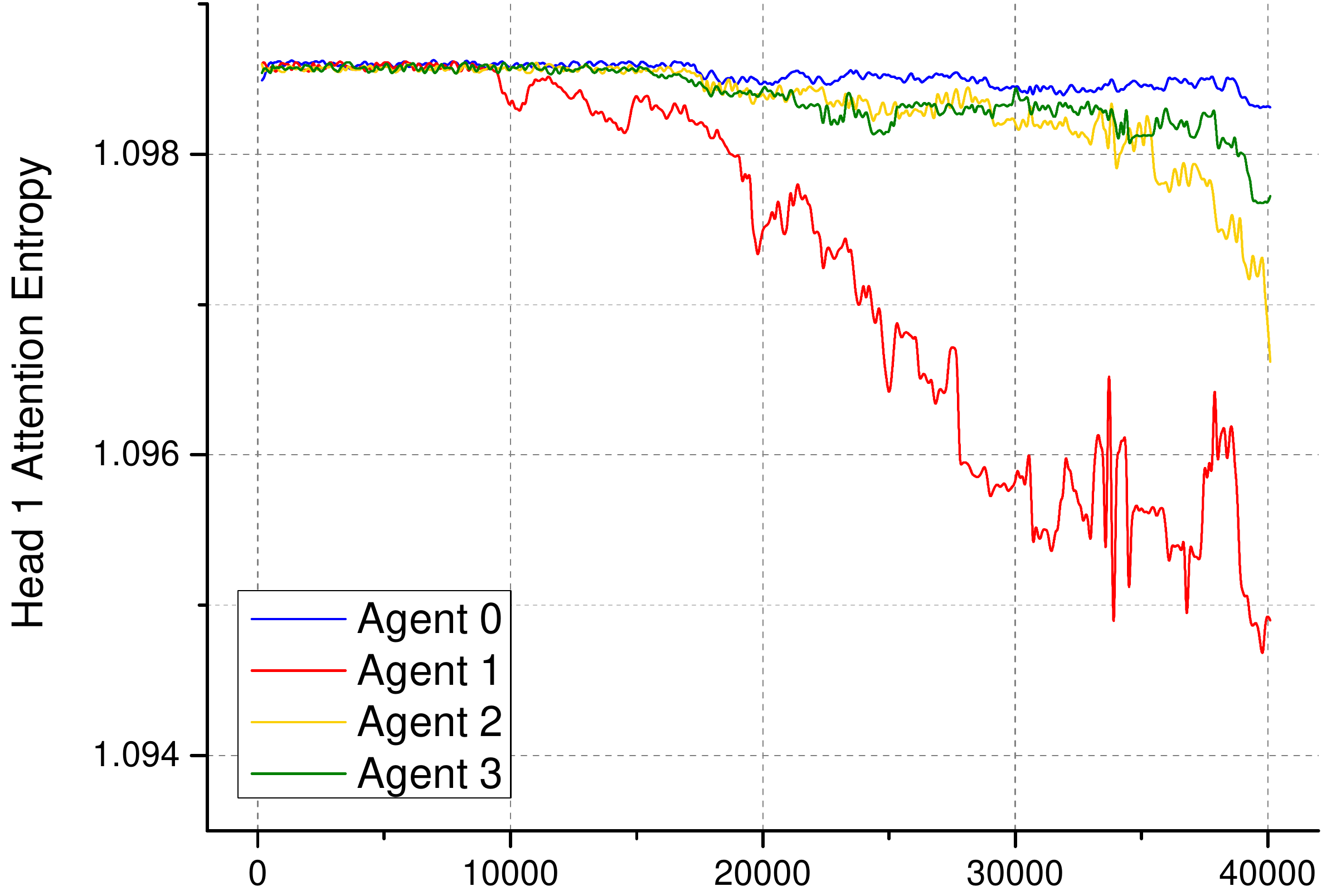}	
		\end{minipage}
	}
	\subfigure[]{
		\begin{minipage}{2.55in}
			\centering
			\includegraphics[width=2.55in,height=1.9in]{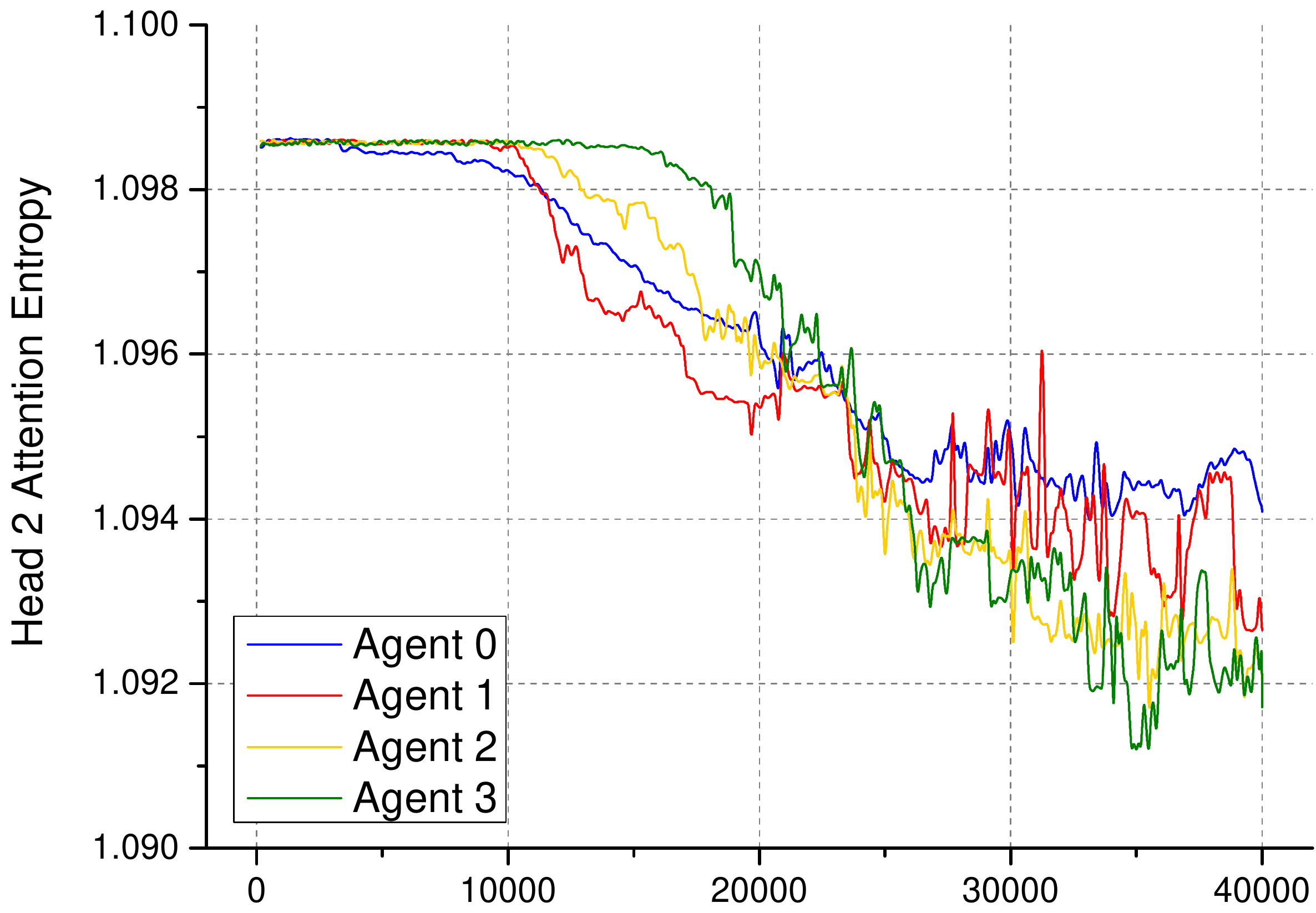}	
		\end{minipage}
	}
	\subfigure[]{
		\begin{minipage}{2.55in}
			\centering
			\includegraphics[width=2.55in,height=1.9in]{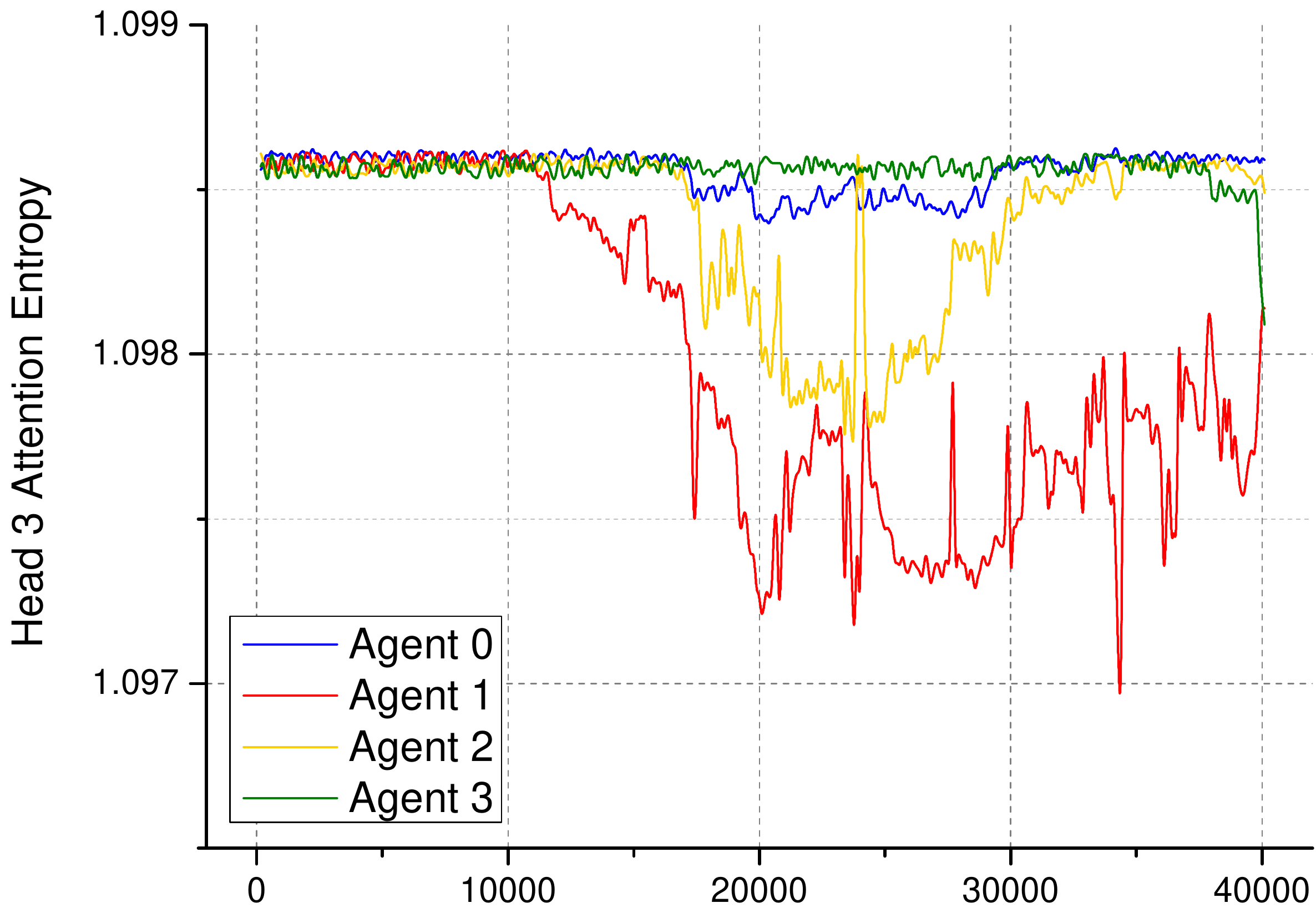}	
		\end{minipage}
	}	
	\caption{Attention ``entropy'' for each head of four agents over the course of training in the multi-agent environment.}	
\end{figure*} 

\begin{figure*}[htbp]
	\begin{minipage}[t]{0.45\linewidth}
		\centering
		\includegraphics[height=5.8cm,width=5.8cm]{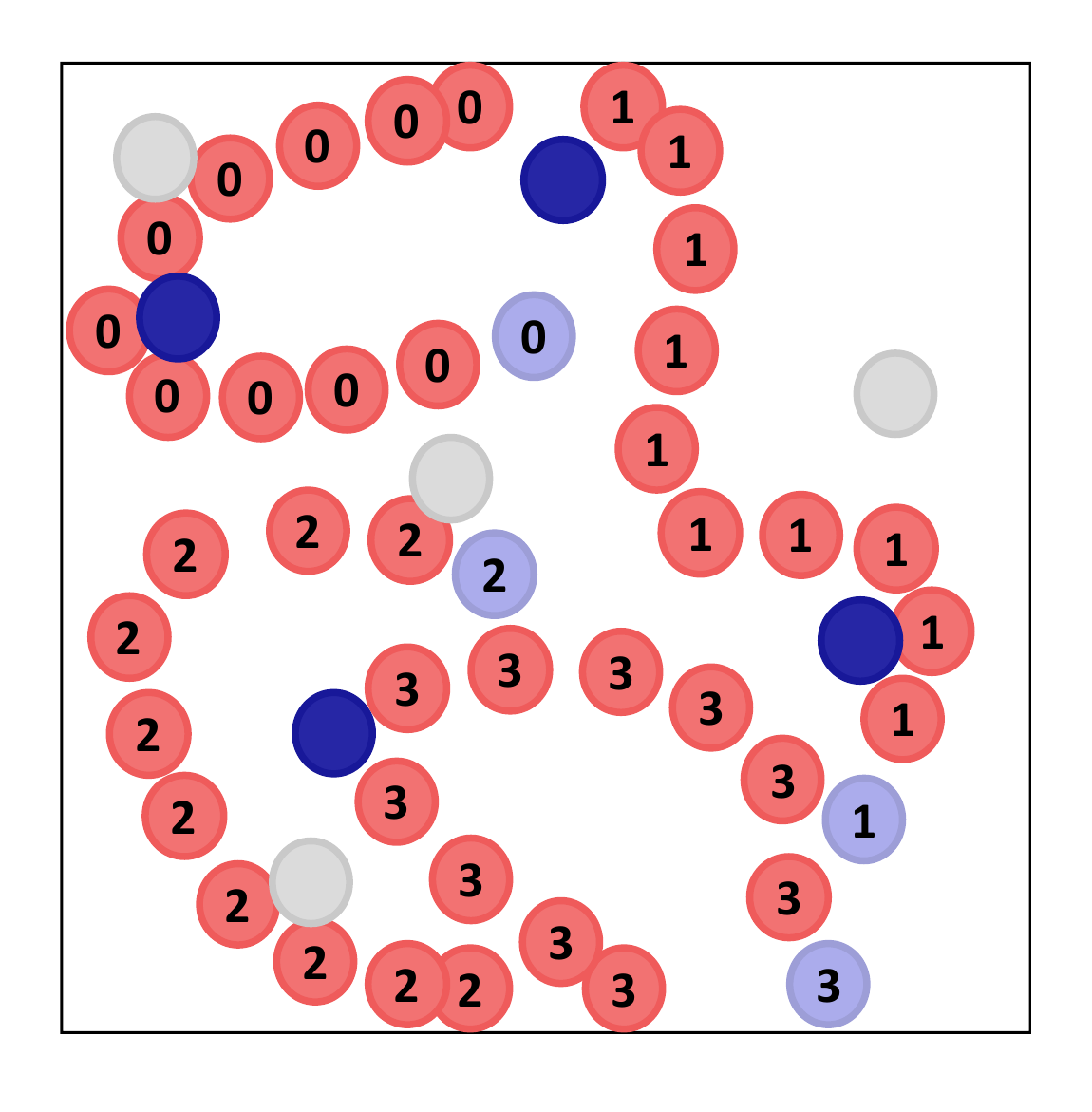}
	\end{minipage}%
	\begin{minipage}[t]{0.5\linewidth}
		\centering
		\includegraphics[height=6cm,width=7cm]{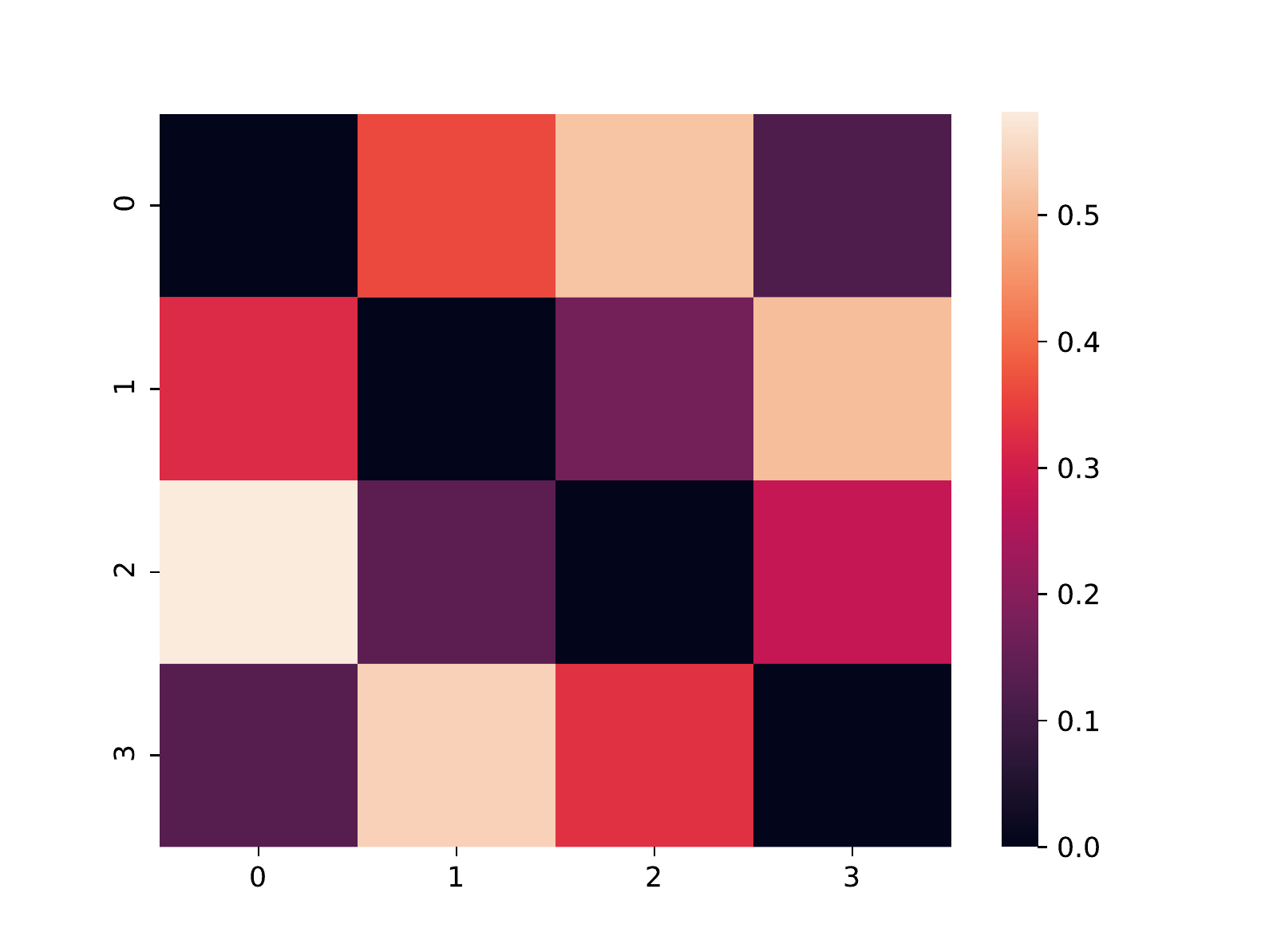}
	\end{minipage}
	\caption{Left: An exploration state of four agents from the last training episode. Right: The corresponding heatmap of attention weight among each agent.}
\end{figure*} 

Since obstacles appear randomly in the training process and the topography distribution of each part of the map is different, each agent faces various difficulties and gets the independent reward at every time step. In addition, each of the four attention heads uses a separate set of parameters to determine an aggregated contribution from all other agents, which means each agent tends to be influenced differently by other agents, so it is reasonable that each agent uses a different combination of four attention heads. 

As shown in figure 6, each agent mostly uses the attention head 2, which indicates that the agents' observation and action information focused by attention head 2 assists more in the exploration task. However, as for agent 1, it needs the main participation of both attention head 2 and head 1 during the training process. As a result, it is obvious that all four attention heads are necessary due to the different concerns about agents' information.

In order to analyze the impact of attention mechanism more, we consider the attention entropy of each attention head for the four agents (Figure 7). Similarly, each head focuses on the different agents at every time step in the training process and focuses a different combination of four agents, which is the same conclusion as the one shown above. It is clear that each head has a different emphasis on agents' observation and action information determined by a specific set of parameters. For instance, head 0, head 1 and head 3 prefer to focus more on the information of agent 1 later in the training phase, while head 2 gives roughly the same concern on all the agents. Besides, each head tends to give a large focus on the information of agent 1, which can also be seen from Figure 6 that all the four heads are used a lot by agent 1.

%%%%%%%%%%%%%%%%%%%%%%%%%%%%%%%%%%%%%%%%%%
To investigate the correlation between the attention weight and the state between agents, we further pick a special state from the last training epoch that could explain the optimization ability of the attention mechanism in MAMECS. The exploration state of four agents from the last training episode (left) and the corresponding heatmap of attention weight among each agent (right) is illustrated in Fig. 8. The regions that have higher attention weight are lighter in color, and the sum of attention weight of each agent is 1 due to the normalization.

Generally, there is larger attention weight between agents with closer distance, like agent 0 and 2, agent 1 and 3. However, regarding the agents far from the current agent, whose trajectory area tends to be explored by the current agent obtains the higher weight. To be specific, as for agent 0, the attention weight between agent 2 $0.52$ is higher than that between agent 1 $0.35$, and they are both higher than that between agent 3 $0.13$, which could illustrate the effect of the attention mechanism. Agent 2 is closest to agent 0, which leads to the highest weight. Although agent 0 is far from both agent 1 and agent 3, it is going to explore the trajectory area of agent 1, so agent 0 will pay more attention to the information of agent 1 rather than agent 3. Therefore, our multi attention heads have learned exactly what we expect. 

%%%%%%%%%%%%%%%%%%%%%%%%%%%%%%%%%%%%%%%%%%

\section{Conclusion}
\label{sec:conclusion}

This paper proposes an multi-head attention based training policies for multi-robot exploration task, MAMECS. The key idea is to utilize multi-head attention mechanism to select meaningful information between related agents for estimating critics. Evaluations on the task of multi-robot exploration clearly show the model outperforms the recently proposed approaches: MADDPG and COMA. MAMECS can obtain higher average rewards and improve exploration performance. We also analyze the attention weight to illustrate the function of each attention head. 

In our future work, we will compare the performance of MAMECS with other baseline methods in Predator and Prey scenario. Besides, we will increase the number of agents and further highlight the advantage of cooperation ability in multi-agent reinforcement learning systems. 

\section{Acknowledgment} This work was supported by the National Natural Science Foundation of China (Grant Numbers
61751208, 61502510, and 61773390), the Outstanding Natural Science Foundation of Hunan Province (Grant
Number 2017JJ1001), and the Advanced Research Program (No. 41412050202).

\section{Reference}
\label{sec:reference}

\appendix
\section{}
Considering there is no known way to find optimal solutions for the problem of fully cover a square with minimum amount of radius circles, we have proposed a theorem to achieve higher coverage ratio. 

\begin{theorem}
	Arrangement of circles tangent to each other does not necessarily lead to the maximum coverage ratio.
\end{theorem}

%% Example of a proof:
\begin{proof}
	We prove this theorem by an example, considering the coverage problem of 4 circles, the coverage ratio $P_c$ is calculated as the ratio of circles' union area $S_{union}$ to the area of its circumscribed square $S_{A^{'}B^{'}C^{'}D^{'}}$. The circles are tangent to the edge of circumscribed square. The following example illustrates 4 same circles with radius $r$ arranged in two patterns. 
	\begin{figure*}[hbtp]
		\centering
		\subfigure[4 circles tangent to each other.]{
			\begin{minipage}{2.55in}
				\centering
				\includegraphics[width=2.55in,height=2.55in]{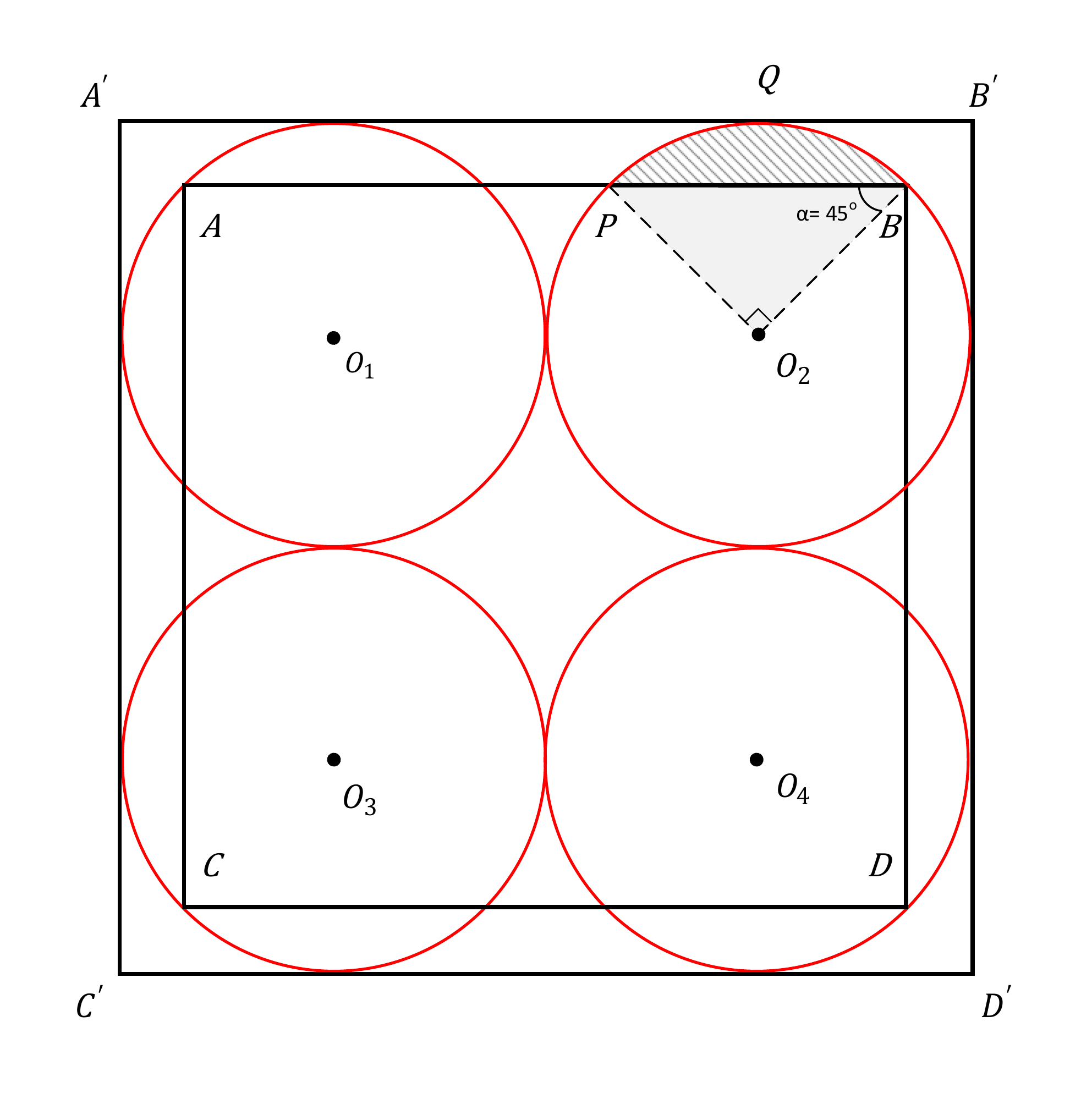}	
			\end{minipage}
		}
		\subfigure[4 circles intersect in a certain pattern.]{
			\begin{minipage}{2.56in}
				\centering
				\includegraphics[width=2.56in,height=2.55in]{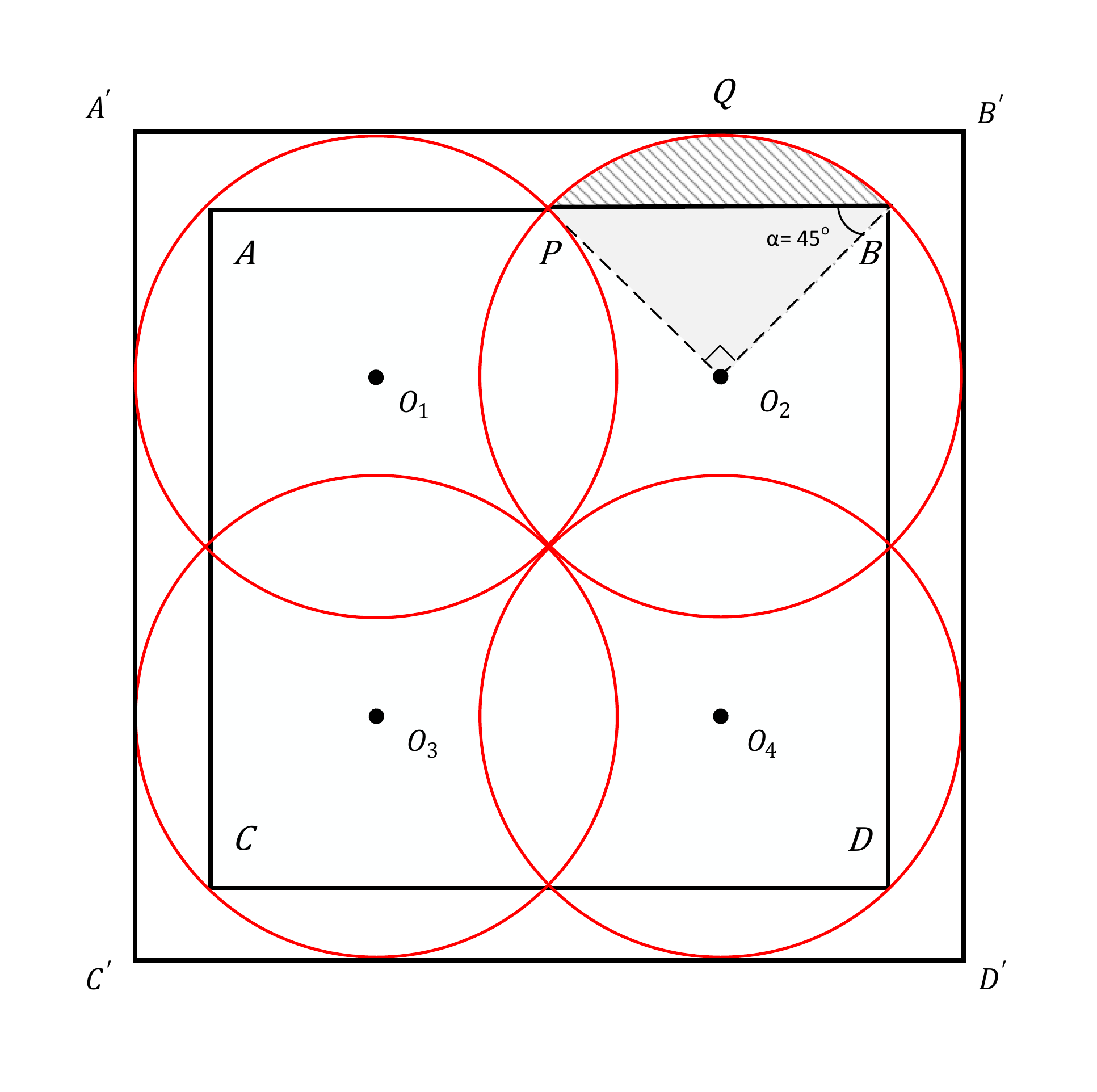}	
			\end{minipage}
		}
	
		\caption{The diagram of 4 circles arranged in different patterns.}	
	\end{figure*}

	As for the circumstance of 4 circles tangent to each other shown in Fig. A.9(a), due to the tangent and symmetric relation, 
	\begin{equation}
	\left\{
	\begin{aligned}
	&S_{union1} = 4\cdot S_{circle}\\		
	&S_{circle} = \pi \cdot r^{2}	\\	
	&S_{square1} = len(A^{'}B^{'})^2 = (4r)^2\\
	\end{aligned}		
	\right.
	\end{equation}
	So, the coverage ratio $P_{c1} = S_{union1}/S_{square1} = \frac{\pi}{4} \approx 0.7854$ .
	
	As for the other circumstance of 4 circles intersect in a certain pattern in Fig. A.9(b), there are
	\begin{equation}
	\left\{
	\begin{aligned}
	&S_{union2} =  S_{ABCD} + 8\cdot S_{PQB}\\		
	&S_{ABCD} = len(AB)^2 = len(2\cdot PB)^2 = (\frac{2r}{\cos \alpha})^2	\\
	&S_{PQB} = S_{O_{2}PQB}	- S_{O_{2}PB} = \frac{\angle PO_{2}B }{2\pi}\cdot \pi r^2 - \frac{1}{2}r^2\\& =  \frac{1}{4}\pi r^2 - \frac{1}{2}r^2\\
	&S_{square2} = len(A^{'}B^{'})^2\\
	&len(A^{'}B^{'}) = len(O_{1}O_{2})+2\cdot r = 2r + \frac{r}{\cos \alpha}
	\end{aligned}		
	\right.
	\end{equation}
	So, the coverage ratio $P_{c2} = S_{union2}/S_{square2} = \frac{4 + 2(\pi -2)\cdot \cos^{2}\alpha}{(1+2\cos \alpha)^2} \approx 0.8822$. $P_{c2}$ is nearly $10\%$ higher than the coverage ratio $P_{c1}$, which means the arrangement of circles tangent to each other does not necessarily lead to the maximum coverage ratio. Therefore, the theorem has been proved by this example.	
\end{proof}

%%%%%%%%%%%%%%%%%%%%%%%%%%%%%%%%%%%%%%%%%%
\end{document}